\documentclass[10pt,a4paper]{article}
\usepackage[utf8]{inputenc}
\usepackage{authblk}
\usepackage[top=25truemm,bottom=25truemm,left=25truemm,right=25truemm,footskip=10truemm]{geometry}
\usepackage{booktabs}
\usepackage{amsmath}
\usepackage{amsthm}
\usepackage{float}
\usepackage{amssymb}
\usepackage{amsfonts}
\usepackage{ascmac}
\usepackage{scalefnt}
\usepackage{indentfirst}
\usepackage{setspace}
\usepackage{subcaption}
\usepackage{lscape}
\usepackage{natbib}
\usepackage{url}
\usepackage{color}
\DeclareMathOperator*{\argmin}{arg\,min}

\newtheorem{theorem}{Theorem}
\newtheorem{proposition}[theorem]{Proposition}
\newtheorem{corollary}[theorem]{Corollary}
\newtheorem{lemma}[theorem]{Lemma}
\newtheorem{assumption}{Assumption}

\newtheorem{remark}{Remark}

\title{
A Note on Estimation Error Bound and Grouping Effect of Transfer Elastic Net
}

\author[]{Yui Tomo}
\affil[]{Center for Surveillance, Immunization, and Epidemiologic Research, National Institute of Infectious Diseases, 1-23-1 Toyama, Shinjuku-Ku, Tokyo 162-0052, Japan
\\E-mail: tomoy@niid.go.jp}

\date{}

\begin{document}

\maketitle
\begin{abstract} 
The Transfer Elastic Net is an estimation method for linear regression models that combines $\ell_1$ and $\ell_2$ norm penalties to facilitate knowledge transfer.
In this study, we derive a non-asymptotic $\ell_2$ norm estimation error bound for the estimator and discuss scenarios where the Transfer Elastic Net effectively works.
Furthermore, we examine situations where it exhibits the grouping effect, which states that the estimates corresponding to highly correlated predictors have a small difference.
\end{abstract} 


\section{Introduction}

Regularized estimation of regression models involves estimation methods to achieve specific desirable properties in the resulting estimators through penalization terms added to the ordinary loss function.
The Least Absolute Shrinkage and Selection Operator (Lasso) promotes sparsity in the coefficient estimates by applying a penalty based on the $\ell_1$ norm of the parameters, leading to some coefficients being exactly zero \citep{tibshirani1996regression}.
Similarly, the Elastic Net encourages sparse estimates via the combined penalty of the 
$\ell_1$ norm and $\ell_2$ norm \citep{zou2005regularization}.
Thanks to the $\ell_2$ norm penalty, the Elastic Net exhibits the grouping effect, wherein highly correlated predictors tend to be assigned coefficient estimates with small differences \citep{zou2005regularization,zhou2013grouping}.
Therefore, the Elastic Net effectively addresses correlations among predictors. 

Transfer learning refers to a set of techniques to utilize the knowledge acquired from solving source problems to solve another target problem.
Recently, \cite{takada2020transfer} proposed the Transfer Lasso, which uses the $\ell_1$ norm to transfer the information of the source estimates to the target problem, which encourages the sparsity in the estimates and the sparsity in the changes of the estimates.
\cite{tomo2024transfer} proposed the Transfer Elastic Net to address high-dimensional features in the field of bioinformatics, employing a combination of $\ell_1$ and $\ell_2$ norms with the expectation of mitigating instability caused by correlations among variables.
Although \cite{colangelo2021essays} proposed a similar estimator, they do not use their $\ell_2$ norm for knowledge transfer.
In the following section, we define the estimator of the Transfer Elastic Net proposed in \cite{tomo2024transfer}.

Suppose that we have responses $y_{i} \in \mathbb{R}$ ($i=1,\ldots,n$) and the predictors $\boldsymbol{X}_{i} = (X_{i, 0},\ldots, X_{i, p-1})^{\top} \in \mathbb{R}^{p}$ ($i=1,\ldots,n$) corresponding to each response.
The responses are centered as $\sum_{i=1}^n y_i=0$ and the predictors are standardized such that $\sum_{i=1}^n X_{i, j}=0$ and $\sum_{i=1}^n X_{i, j}^{2} / n=1$ for $j=0, \ldots, p-1$.
We assume that the responses are generated according to a linear regression model with coefficients parameter ${\boldsymbol{\beta}}=({\beta}_{0},\ldots,{\beta}_{p-1})^{\top}\in \mathbb{R}^{p}$.
Specifically, let $\boldsymbol{\beta}^{*} \in \mathbb{R}^{p}$ denote the true parameter vector, and we assume the responses are generated as
\begin{align*}
    y_i
    =
    {\boldsymbol{\beta}^{*}}^{\top} \boldsymbol{X}_{i} + \varepsilon_{i}, \quad \text{($i=1,\ldots,n$)},
\end{align*}
where $\varepsilon_{i} \in \mathbb{R}$ ($i=1,\ldots,n$) are the error terms with their expected values of $0$. 
Assume that $\boldsymbol{\beta}^{*} \in \mathbb{R}^{p}$ has exactly $s$ non-zero elements ($0 \leq s \leq p$).
Let $\boldsymbol{y} = \left(y_{1},\ldots,y_{n}\right)^{\top} \in \mathbb{R}^{n}$ and $\mathbf{X}=(\boldsymbol{X}_{1},\ldots,\boldsymbol{X}_{n})^{\top}\in \mathbb{R}^{n \times p}$.
Let $\tilde{\boldsymbol{\beta}} \in \mathbb{R}^{p}$ denote the parameter estimate obtained from the source problem.
Then, the loss function of the Transfer Elastic Net is defined as:
\begin{align}
    \label{eq:regularized_risk_tenet}
    \mathcal{L}(\boldsymbol{\beta} ; \tilde{\boldsymbol{\beta}})
    &:= \frac{1}{2n} \sum_{i=1}^n \left(y_i - \boldsymbol{\beta}^{\top} \boldsymbol{X}_i\right)^2
    +
    \lambda R(\boldsymbol{\beta},\tilde{\boldsymbol{\beta}};\alpha,\rho)
    ,
    \\
    R(\boldsymbol{\beta},\tilde{\boldsymbol{\beta}};\alpha,\rho)
    &:= \alpha \left\{\rho \|\boldsymbol{\beta}\|_1 + (1-\rho) \|\boldsymbol{\beta}\|_2^2\right\}
    +
    (1-\alpha) \left\{\rho \|\boldsymbol{\beta} - \tilde{\boldsymbol{\beta}}\|_1 + (1-\rho) \|\boldsymbol{\beta} - \tilde{\boldsymbol{\beta}}\|_2^2\right\}
    ,
    \nonumber
\end{align}
where $\lambda \in[0,+\infty)$, $\alpha \in[0,1]$, and $\rho \in[0,1]$ are the tuning parameters, which control the intensity of the regularization, the balance of the $\ell_1$ and $\ell_2$ norms, and the extent of transferring information, respectively.
The estimator $\hat{\boldsymbol{\beta}}_{\mathrm{TENet}}$ is defined as the minimizer of the loss function:
\begin{align*}
    \hat{\boldsymbol{\beta}}_{\mathrm{TENet}}
    &:= \underset{\boldsymbol{\beta} \in \mathbb{R}^{p}}{\operatorname{argmin}} \mathcal{L}(\boldsymbol{\beta} ; \tilde{\boldsymbol{\beta}})
    .
\end{align*}

\cite{takada2020transfer} introduced the generalized restricted eigenvalue condition and derived a non-asymptotic estimation error bound of the Transfer Lasso under the assumption of the sub-Gaussian error terms and the condition. 
In this study, we extend this analysis by establishing an estimation error bound for the Transfer Elastic Net.
Then, we compare this bound with those of the ordinary Elastic Net and the Transfer Lasso to discuss the scenario where the Transfer Elastic Net is recommended to be applied.
Furthermore, we derive an upper bound of the magnitude of the difference in estimates assigned to two different predictors and examine the situation where a high correlation between the predictors results in a small difference in the estimates.

\section{Estimation Error Bound}

Let $\boldsymbol{\Delta} = \tilde{\boldsymbol{\beta}} - \boldsymbol{\beta}^{*}$ and $\boldsymbol{\Delta}_{\alpha} = \boldsymbol{\Delta} - \alpha \tilde{\boldsymbol{\beta}}$ ($={(1-\alpha)\tilde{\boldsymbol{\beta}} - \boldsymbol{\beta}}^{*}$).
Let $\boldsymbol{v}_{S}$ denote the subvector of $\boldsymbol{v} \in \mathbb{R}^{p}$ obtained by extracting the entries indexed by the index set $S$.
We assume that $\boldsymbol{\beta}_S^* \in \mathbb{R}^s$ is non-zero in all its elements, and $\boldsymbol{\beta}_{S^c}^* \in \mathbb{R}^{p-s}$ is the zero vector.
We assume the following conditions to derive an estimation error bound for the Transfer Elastic Net.

\begin{assumption}[Sub-Gaussianity of the error terms]
    \label{ast:subgaussian}
    The error terms $\varepsilon_{1},\ldots,\varepsilon_{n}$ satisfy the conditions below.
    \begin{itemize}
        \item[i.] $\varepsilon_{1},\ldots,\varepsilon_{n}$ are independent and identically distributed samples from a probability distribution $P_{\varepsilon}$ with its expected value of $0$.
        \item[ii.] $\varepsilon_{1},\ldots,\varepsilon_{n}$ are sub-Gaussian with variance proxy $\sigma^{2}$, i.e., for $i = 1,\ldots,n$,
        \begin{align*}
            \mathrm{E}_{\varepsilon_{i} \sim P_{\varepsilon}}[\exp (t \varepsilon_{i})]
            &\leq
            \exp \left(\frac{\sigma^2 t^2}{2}\right), \quad \forall t \in \mathbb{R}.
        \end{align*}
    \end{itemize}
\end{assumption}

\begin{assumption}[Generalized restricted eigenvalue condition]
    \label{ast:gre}
    We define $\mathcal{B}(\alpha, \rho, c, \boldsymbol{\Delta}) \subset \mathbb{R}^{p}$ as
    \begin{align*}
        \mathcal{B}(\alpha, \rho, c, \boldsymbol{\Delta})
        :=
        \{\boldsymbol{v} \in \mathbb{R}^p
        &: (\alpha \rho - c)\|\boldsymbol{v}_{S^{c}}\|_{1} + (1-\alpha)\rho\|\boldsymbol{v}-\boldsymbol{\Delta}\|_{1}
        \\
        &\leq
        (\alpha\rho + c)\|\boldsymbol{v}_{S}\|_{1} + (1-\alpha)\rho \|\boldsymbol{\Delta}\|_{1} + 2(1-\rho)\|\boldsymbol{\Delta}_{\alpha}\|_{2}\|\boldsymbol{v}\|_{2}\},
    \end{align*}
    for $\alpha \in [0,1]$, $\rho \in [0,1]$, $\boldsymbol{\Delta} \in \mathbb{R}^{p}$, and $c>0$.
    Then, the generalized restricted eigenvalue condition in Definition 1 of \cite{takada2020transfer} holds for $\mathcal{B}=\mathcal{B}(\alpha, \rho, c, \boldsymbol{\Delta})$ for all $\alpha \in [0,1]$, $\rho \in [0,1]$, $\boldsymbol{\Delta} \in \mathbb{R}^{p}$, and $c>0$, i.e.,
    \begin{align*}
        \phi(\mathcal{B}):=\inf _{\boldsymbol{v} \in \mathcal{B}} \frac{\boldsymbol{v}^{\top} \frac{1}{n} \mathbf{X}^{\top} \mathbf{X} \boldsymbol{v}}{\|\boldsymbol{v}\|_2^2}>0
        .
    \end{align*}
\end{assumption}

Under these assumptions, we establish Theorem 1, which provides a non-asymptotic estimation error bound for the Transfer Elastic Net.
\begin{theorem}
    \label{thm:estimation_error_tenet}
    Suppose that Assumption \ref{ast:subgaussian} and Assumption \ref{ast:gre} are satisfied.
    We define 
    \begin{align*}
        U_{\mathrm{TENet}}
        &:=
        \frac{(\alpha \rho+c) \lambda \sqrt{s}+2 \lambda (1-\rho)\left\|\boldsymbol{\Delta}_{\alpha}\right\|_2+\sqrt{D}}{2 \lambda (1-\rho)+\phi_{\mathrm{TENet}}}
        ,
        \\
        D&=\left\{(\alpha \rho+c) \lambda \sqrt{s}+2 \lambda \alpha(1-\rho)\left\|\boldsymbol{\Delta}_{\alpha}\right\|_2\right\}^2
        \\
        &\quad+2 \lambda(1-\alpha) \rho\|\boldsymbol{\Delta}\|_1 \{2 \lambda (1-\rho)+\phi_{\mathrm{TENet}}\}
        ,
    \end{align*}
    where $\phi_{\mathrm{TENet}} = \phi(\mathcal{B}(\alpha, \rho, c, \boldsymbol{\Delta}))$.
    Then we have
    \begin{align*}
        \left\|\hat{\boldsymbol{\beta}}_{\mathrm{TENet}}-\boldsymbol{\beta}^*\right\|_2
        &\leq
        U_{\mathrm{TENet}}
        ,
    \end{align*}
    with probability at least $1-\exp(-n c^2 \lambda^2 /2\sigma^2 + \log(2p))$.
\end{theorem}

Subsequently, we consider the ordinary Elastic Net and the Transfer Lasso.
The estimator of the Elastic Net $\hat{\boldsymbol{\beta}}_{\mathrm{ENet}}$ is defined as
\begin{align*}
    \hat{\boldsymbol{\beta}}_{\mathrm{ENet}} := \argmin_{\beta\in\mathbb{R}^{p}}\left\{\frac{1}{2n} \sum_{i=1}^n \left(y_i - \boldsymbol{\beta}^{\top} \boldsymbol{X}_i\right)^2
    +
    \lambda R(\boldsymbol{\beta},{\boldsymbol{0}};1,\rho)
    \right\}
    ,
\end{align*}
and, similarly, the estimator of the Transfer Lasso $\hat{\boldsymbol{\beta}}_{\mathrm{TLasso}}$ is defined as
\begin{align*}
    \hat{\boldsymbol{\beta}}_{\mathrm{TLasso}} := \argmin_{\beta\in\mathbb{R}^{p}}\left\{\frac{1}{2n} \sum_{i=1}^n \left(y_i - \boldsymbol{\beta}^{\top} \boldsymbol{X}_i\right)^2
    +
    \lambda R(\boldsymbol{\beta},\tilde{\boldsymbol{\beta}};\alpha,1)
    \right\}
    .
\end{align*}
The estimation error for $\hat{\boldsymbol{\beta}}_{\mathrm{ENet}}$ and $\hat{\boldsymbol{\beta}}_{\mathrm{TLasso}}$ can be immediately obtained from Theorem \ref{thm:estimation_error_tenet} by setting $\alpha=1$ and $\rho=1$, respectively.
\begin{corollary}
    \label{col:col1}
    Suppose that Assumption \ref{ast:subgaussian} and Assumption \ref{ast:gre} are satisfied.
    We define  
    \begin{align*}
        U_{\mathrm{ENet}}
        &:=
        \frac{2(\rho+c) \lambda \sqrt{s}+4 \lambda (1-\rho)\left\|\boldsymbol{\beta}^{*}\right\|_2}{2 \lambda (1-\rho)+\phi_{\mathrm{ENet}}}
        ,
        \\
        U_{\mathrm{TLasso}}
        &:=
        \frac{(\alpha+c) \lambda \sqrt{s}+\sqrt{
        \left\{(\alpha+c) \lambda \sqrt{s}\right\}^2+2 \lambda(1-\alpha) \|\boldsymbol{\Delta}\|_1 \phi_{\mathrm{TLasso}}
        }}{\phi_{\mathrm{TLasso}}}
        ,
    \end{align*}
    where $\phi_{\mathrm{ENet}} = \phi(\mathcal{B}(1, \rho, c, \boldsymbol{0}))$ and $\phi_{\mathrm{TLasso}} = \phi(\mathcal{B}(\alpha, 1, c, \boldsymbol{\Delta}))$.
    Then we have 
    \begin{align*}
        \left\|\hat{\boldsymbol{\beta}}_{\mathrm{ENet}}-\boldsymbol{\beta}^*\right\|_2
        &\leq
        U_{\mathrm{ENet}}
        ,
    \end{align*}
    and
    \begin{align*}
        \left\|\hat{\boldsymbol{\beta}}_{\mathrm{TLasso}}-\boldsymbol{\beta}^*\right\|_2
        &\leq
        U_{\mathrm{TLasso}}
        ,
    \end{align*}
    with probability at least $1-\exp(-n c^2 \lambda^2 /2\sigma^2 + \log(2p))$.
\end{corollary}

\begin{remark}
    The estimation error of Transfer Lasso described in Corollary \ref{col:col1} is consistent with the result of \cite{takada2020transfer}.
\end{remark}

From Theorem \ref{thm:estimation_error_tenet} and Corollary \ref{col:col1}, we establish the propositions that provide the comparison of the bounds.
We assume that the source problem is highly related to the target problem and that $\tilde{\boldsymbol{\beta}}=\boldsymbol{\beta}^{*}$. 
Then, we prove the following proposition providing a relationship between the error bounds of the Transfer Elastic Net and the ordinary Elastic Net for any $\alpha$ and the same $\lambda$ and $\rho$.

\begin{proposition}
    \label{prop:tenet_vs_enet}
    Suppose $\tilde{\boldsymbol{\beta}}=\boldsymbol{\beta}^{*}$. Then we have $U_{\mathrm{ENet}} \geq U_{\mathrm{TENet}}$.
\end{proposition} 

We then consider the Transfer Lasso.
The following proposition shows a sufficient condition under which $U_{\mathrm{TENet}}$ is guaranteed to be smaller than $U_{\mathrm{TLasso}}$ for any $\rho$ and the same $\lambda$ and $\alpha$.

\begin{proposition}
    \label{prop:tenet_vs_tlasso}
    Suppose $\tilde{\boldsymbol{\beta}}=\boldsymbol{\beta}^{*}$. If $\sqrt{s}/2 \geq \|\boldsymbol{\beta}^{*}_{S}\|_2$ and 
    $\phi_{\mathrm{TENet}} + 2 \lambda (1-\rho) \geq \phi_{\mathrm{TLasso}}$
    , then we have $U_{\mathrm{TLasso}} \geq U_{\mathrm{TENet}}$.
\end{proposition}

These propositions suggest that the Transfer Elastic Net can achieve a lower estimation error bound than those of the Elastic Net and the Transfer Lasso when the source problem is highly related to the target problem.
In particular, when the correlation between predictors is large, $\phi$ may take a value close to zero.
Even in such a case, the error bound for the Transfer Elastic Net can be smaller than that of the Transfer Lasso if a sufficiently large $\lambda$ value is chosen.

Now we discuss the conditions under which the generalized restricted eigenvalue condition is satisfied.
From the result of \cite{raskutti2010restricted}, the ordinary restricted eigenvalue condition for the Lasso is satisfied with high probability when predictors are independent and identically distributed (i.i.d.) samples from a Gaussian distribution with a covariance matrix that satisfies the condition.
For the generalized restricted eigenvalue condition, we prove the following proposition.

\begin{proposition}
    \label{prop:sample_X_GRE}
    Suppose $2 \alpha \rho - c - \rho > 0$.
    Suppose $\boldsymbol{X}_{1},\ldots,\boldsymbol{X}_{n}$ are $i.i.d.$ samples from $\mathcal{N}(\boldsymbol{0},\mathbf{\Sigma})$ with $\mathbf{\Sigma} \in \mathbb{R}^{p \times p}$ and the following inequality holds for some $\gamma > 0$:
    \begin{align*}
        {\boldsymbol{v}^{\top} \mathbf{\Sigma} \boldsymbol{v}}
        >
        \gamma \|\boldsymbol{v}\|_2^2, \quad \text{for all}\, \boldsymbol{v} \in \mathcal{B}(\alpha,\rho,c,\boldsymbol{\Delta})
        .
    \end{align*}
    Under these conditions, there exist universal constants $c^{\prime}, c^{\prime \prime}, c^{\prime \prime \prime}>0$ such that if the sample size satisfies
    \begin{align*}
        n
        &>
        \frac{c^{\prime \prime \prime} M}{\gamma} {\left(\frac{2 \alpha \rho \sqrt{s} + 2(1-\rho)\left\|\boldsymbol{\Delta}_\alpha\right\|_2}{2 \alpha \rho-c-\rho}\right)^2} \log p
        ,
    \end{align*}
    where $M := \max _{j=1, \ldots, p} \Sigma_{j j}$, then, we have 
    \begin{align*}
        \frac{1}{n}{\boldsymbol{v}^{\top} \mathbf{X}^{\top} \mathbf{X} \boldsymbol{v}}
        >
        \frac{\gamma}{64} {\|\boldsymbol{v}\|_2^2}
        \quad \text{for all}\, \boldsymbol{v} \in \mathcal{B}(\alpha,\rho,c,\boldsymbol{\Delta})
        ,
    \end{align*}
    with probability at least $1 - c^{\prime \prime} \exp(-c^{\prime} n)$.
\end{proposition}

This proposition suggests that a broad class of Gaussian matrices satisfies the generalized restricted eigenvalue condition for $\mathcal{B} = \mathcal{B}(\alpha,\rho,c,\boldsymbol{\Delta})$ with high probability.

\section{Grouping Effect}

The grouping effect of the Elastic Net ensures that the estimates of the coefficients corresponding to highly correlated predictors have a small difference \citep{zou2005regularization,zhou2013grouping}.
We show that the Transfer Elastic Net retains a similar property.
Let $r_{jk}$ denote the correlation coefficient between the $j$th and $k$th column of $\mathbf{X}$.
We write $\hat{\boldsymbol{\beta}} = \hat{\boldsymbol{\beta}}_{\mathrm{TENet}}$.
Then we prove the following theorem.

\begin{theorem}
    \label{thm:grouping_effect}
    If $\hat{{\beta}}_{j}\hat{{\beta}}_{k} > 0$ and $(\hat{{\beta}}_{j} - \tilde{{\beta}}_{j})(\hat{{\beta}}_{k}-\tilde{{\beta}}_{k}) > 0$, then, for $\rho \neq 1$ and $\lambda \neq 0$, we have
    \begin{align*}
        \left|\hat{{\beta}}_{j} - \hat{{\beta}}_{k}\right|
        &\leq
        Z\sqrt{1-r_{jk}}
        + (1-\alpha)\left|\tilde{{\beta}}_{j} - \tilde{{\beta}}_{k}\right|
        ,
        \quad
        Z
        =
        \sqrt{\frac{\left\|\boldsymbol{y}\right\|_{2}^{2}
        +
        2n\lambda(1-\alpha)\rho \|\tilde{\boldsymbol{\beta}}\|_1 + 2n\lambda(1-\alpha) (1-\rho) \|\tilde{\boldsymbol{\beta}}\|_2^2}{2n \lambda^{2}(1-\rho)^{2}}}
        .
    \end{align*}
\end{theorem} 

This theorem suggests that the estimates corresponding to strongly correlated variables exhibit a small difference if the corresponding source estimates have a small difference or $\alpha$ is close to $1$.
Although the conditions of the theorem involve estimators, a similar condition is presented in the results of \cite{zou2005regularization}.

Now suppose that the source estimate $\tilde{\boldsymbol{\beta}}$ was obtained by applying the ordinary Elastic Net to the data with the response ${\boldsymbol{y}^{\prime}}\in\mathbb{R}^{m}$ and the predictors ${\mathbf{X}}^{\prime}\in\mathbb{R}^{m \times p}$, with $\sum_{i=1}^m y^{\prime}_i=0$, $\sum_{i=1}^m X_{i, j}^{\prime}=0$, and $\sum_{i=1}^m {X_{i, j}^{\prime}}^{2}/m=1$ for $j=0, \ldots, p-1$.
Therefore, we assume that $\tilde{\boldsymbol{\beta}}$ is defined as:
\begin{align*}
    \tilde{\boldsymbol{\beta}} = \argmin_{\beta\in\mathbb{R}^{p}}\left\{\frac{1}{2m} \sum_{i=1}^m \left(y_i^{\prime} - \boldsymbol{\beta}^{\top} \boldsymbol{X}^{\prime}_i\right)^2
    +
    \lambda^{\prime} R(\boldsymbol{\beta},{\boldsymbol{0}};1,\rho^{\prime})
    \right\}
    .
\end{align*}
Let $\tilde{r}_{jk}$ denote the correlation coefficient between the $j$th and $k$th column of $\mathbf{X}^{\prime}$.
Then we have
\begin{align*}
    \left|\tilde{{\beta}}_{j} - \tilde{{\beta}}_{k}\right|
    &\leq
    \frac{\|\tilde{\boldsymbol{y}}\|_2\sqrt{2m(1-\tilde{r}_{jk})}}{2m\lambda^{\prime}(1-\rho^{\prime})}
    ,
\end{align*}
for $\rho^{\prime} \neq 1$ and $\lambda^{\prime} \neq 0$.
This result suggests that, if there are highly correlated predictors in the source data and the corresponding predictors are also highly correlated in the target data, adopting Elastic Net estimates as the source estimates will ensure a small difference between the corresponding estimates of the Transfer Elastic Net.

\begin{remark}
    When the regularization term is defined as
    \begin{align*}
        R(\boldsymbol{\beta},\tilde{\boldsymbol{\beta}};\alpha,\rho_1,\rho_2)
        &:= \alpha \left\{\rho_1 \|\boldsymbol{\beta}\|_1 + (1-\rho_1) \|\boldsymbol{\beta}\|_2^2\right\}
        +
        (1-\alpha) \left\{\rho_1 \|\boldsymbol{\beta} - \tilde{\boldsymbol{\beta}}\|_1 + \rho_2 \|\boldsymbol{\beta} - \tilde{\boldsymbol{\beta}}\|_2^2\right\}
        ,
    \end{align*}
    and we set $\rho_2 = 0$, then we have 
    \begin{align*}
        \left|\hat{{\beta}}_{j} - \hat{{\beta}}_{k}\right|
        &\leq
        Z\sqrt{1-r_{jk}}
        ,
        \quad
        Z
        =
        \sqrt{\frac{\left\|\boldsymbol{y}\right\|_{2}^{2}
        +
        2n\lambda(1-\alpha)\rho_1 \|\tilde{\boldsymbol{\beta}}\|_1}{2n \alpha^2 \lambda^{2}(1-\rho_1)^{2}}}
        ,
    \end{align*}
    if $\hat{{\beta}}_{j}\hat{{\beta}}_{k} > 0$ and $(\hat{{\beta}}_{j} - \tilde{{\beta}}_{j})(\hat{{\beta}}_{k}-\tilde{{\beta}}_{k}) > 0$,
    according to the same derivation.
    This result suggests that the estimates corresponding to strongly correlated variables exhibit a small difference, regardless of $\alpha \neq 0$, $\rho_1 \neq 1$, and the difference of $\tilde{\beta}_j$ and $\tilde{\beta}_k$ if the knowledge transfer is achieved solely through the $\ell_1$ norm term.
\end{remark}

\section*{Acknowledgements}

The authors thank Takayuki Kawashima for providing valuable comments on an early draft of this manuscript.
The authors are also deeply grateful to Masaaki Takada for helpful discussions and suggestions regarding aspects that were insufficiently addressed.

\renewcommand{\thesection}{Appendix \Alph{section}}
\setcounter{section}{0}
\section{Proof of Theorem \ref{thm:estimation_error_tenet}}

In this section, we write $\hat{\boldsymbol{\beta}} = \hat{\boldsymbol{\beta}}_{\mathrm{TENet}}$ and $\phi = \phi_{\mathrm{TENet}}$.
First, we establish the following lemma.
\begin{lemma}[Technical Lemma]
\label{lem:xe_subgaussian}
Suppose that Assumption \ref{ast:subgaussian} is satisfied.
Then, we have
\begin{align*}
    \frac{1}{n}\left({\mathbf{X}}^{\top} {\boldsymbol{\varepsilon}}\right)^{\top}\left(\hat{\boldsymbol{\beta}}-\boldsymbol{\beta}^*\right)
    \leq
    c\lambda \left\|\hat{\boldsymbol{\beta}}-\boldsymbol{\beta}^*\right\|_{1}
    \quad\text{for some $c > 0$},
\end{align*}
with probability at least $1-\exp(-n c^2 \lambda^2 /2\sigma^2 + \log(2p))$.
\end{lemma}
\begin{proof}
From Hölder's inequality, it follows that 
\begin{align*}
    \frac{1}{n}\left({\mathbf{X}}^{\top} {\boldsymbol{\varepsilon}}\right)^{\top}\left(\hat{\boldsymbol{\beta}}-\boldsymbol{\beta}^*\right)
    \leq
    \frac{1}{n}\left\|{\mathbf{X}}^{\top} {\boldsymbol{\varepsilon}}\right\|_{\infty}
    \left\|\hat{\boldsymbol{\beta}}-\boldsymbol{\beta}^*\right\|_{1}
    .
\end{align*}    
From Assumption \ref{ast:subgaussian}, we have
\begin{align*}
    \mathrm{E}_{\varepsilon_{1},\ldots,\varepsilon_{n} \sim P_{\varepsilon}}\left[\exp \left(t \left(\frac{1}{n}{\mathbf{X}}^{\top} {\boldsymbol{\varepsilon}}\right)_{j}\right)\right]
    &=
    \mathrm{E}_{\varepsilon_{1},\ldots,\varepsilon_{n} \sim P_{\varepsilon}}\left[\exp \left(\frac{t}{n}{\sum_{i=1}^{n}{X_{i,j}}} {{\varepsilon}_{j}}\right)\right]
    \\
    &=
    \prod_{i=1}^{n} \mathrm{E}_{\varepsilon_{1},\ldots,\varepsilon_{n} \sim P_{\varepsilon}}\left[\exp \left(\frac{t}{n}{{X_{i,j}}} {{\varepsilon}_{j}}\right)\right]
    \\
    &\leq
    \prod_{i=1}^{n} \exp \left(\frac{\sigma^{2} \left(t{X_{i,j}/n}\right)^{2}}{2} \right)
    \\
    &=
    \exp \left(\frac{\left(\sigma^{2}\sum_{i=1}^{n}{X_{i,j}^{2}/n^{2}}\right) t^{2}}{2}\right)
    \\
    &=
    \exp \left(\frac{\left(\sigma/\sqrt{n}\right)^{2} t^{2}}{2}\right),
    \quad \forall t \in \mathbb{R}
    .
\end{align*}
Therefore, $\left(\frac{1}{n}{\mathbf{X}}^{\top} {\boldsymbol{\varepsilon}}\right)_{j}$ is sub-Gaussian with variance proxy $\left(\sigma/\sqrt{n}\right)^{2}$ and we then have
\begin{align*}
    P\left(\left|\left(\frac{1}{n}{\mathbf{X}}^{\top} {\boldsymbol{\varepsilon}}\right)_{j}\right| > t\right)
    &\leq
    2 \exp\left(-\frac{t^2}{2 \left(\sigma/\sqrt{n}\right)^{2}}\right)
    ,
    \quad
    \forall t > 0.
\end{align*}
From the above results, we have
\begin{align*}
    P\left(\left\|\frac{1}{n}{\mathbf{X}}^{\top} {\boldsymbol{\varepsilon}}\right\|_{\infty} > t\right)
    &=
    P\left(\bigcup_{j=0}^{p-1} \left|\left(\frac{1}{n}{\mathbf{X}}^{\top} {\boldsymbol{\varepsilon}}\right)_{j}\right| > t\right)
    \\
    &\leq
    \sum_{j=0}^{p-1} P\left(\left|\left(\frac{1}{n}{\mathbf{X}}^{\top} {\boldsymbol{\varepsilon}}\right)_{j}\right| > t\right)
    \\
    &\leq
    \sum_{j=0}^{p-1} 2 \exp\left(-\frac{t^2}{2 \left(\sigma/\sqrt{n}\right)^{2}}\right)
    \\
    &=
    \exp\left(-\frac{n t^2}{2 \sigma^{2}} + \log(2p)\right),
    \quad
    \forall t > 0.
\end{align*}
Substituting $t$ for $c\lambda$ with some $c>0$ and subtracting both sides from $1$ completes the proof.
\end{proof}

As a preparation, the loss function of the Transfer Elastic Net (\ref{eq:regularized_risk_tenet}) can be expressed as
\begin{align*}
    L(\boldsymbol{\beta}, \tilde{\boldsymbol{\beta}})
    & =
    \frac{1}{2 n}\|\boldsymbol{y}-\mathbf{X}\boldsymbol{ \beta}\|_2^2 +
    \lambda \alpha \rho\|\boldsymbol{ \beta}\|_1+\lambda \alpha(1-\rho)\|\boldsymbol{ \beta}\|_2^2 +
    \lambda(1-\alpha) \rho\|\boldsymbol{\beta}-\tilde{\boldsymbol{\beta}}\|_1+\lambda(1-\alpha)(1-\rho)\|\boldsymbol{\beta}-\tilde{\boldsymbol{\beta}}\|_2^2 \\
    & =
    \frac{1}{2 n}\|\tilde{\boldsymbol{y}}-\widetilde{\mathbf{X}} \boldsymbol{\beta}\|_2^2
    +
    \lambda \alpha \rho\|\boldsymbol{\beta}\|_1+\lambda(1-\alpha) \rho\|\boldsymbol{\beta}-\tilde{\boldsymbol{\beta}}\|_1-2\lambda(1-\alpha)(1-\rho){\boldsymbol{\beta}}^{\top}\tilde{\boldsymbol{\beta}} + \lambda(1-\alpha)(1-\rho)\|\tilde{\boldsymbol{\beta}}\|_2^2
    ,
\end{align*}
where $\tilde{\boldsymbol{y}}=(\boldsymbol{y}^{\top},0)^{\top} \in \mathbb{R}^{n+p}$ and $\widetilde{\mathbf{X}}=(\mathbf{X}^{\top},\mathbf{I}\sqrt{2 n \lambda  (1 - \rho)})^{\top} \in \mathbb{R}^{(n+p) \times p}$ with a p-dimensional zero-vector $\boldsymbol{0}$ and $p \times p$ identity matrix $\mathbf{I}$.
Using this expression, we have
\begin{align*}
    L(\boldsymbol{\beta}^{*}, \tilde{\boldsymbol{\beta}})
    & =
    \frac{1}{2 n}\|\tilde{\boldsymbol{y}}-\widetilde{\mathbf{X}} \boldsymbol{\beta}^{*}\|_2^2
    +
    \lambda \alpha \rho\|\boldsymbol{\beta}^{*}\|_1+\lambda(1-\alpha) \rho\|\boldsymbol{\beta}^{*}-\tilde{\boldsymbol{\beta}}\|_1+\lambda(1-\alpha)(1-\rho)\|\boldsymbol{\beta}^{*}-\tilde{\boldsymbol{\beta}}\|_2^2
    \\
    & =
    \frac{1}{2 n}\|\tilde{\boldsymbol{\varepsilon}}\|_2^2
    +
    \lambda \alpha \rho\|\boldsymbol{\beta}^{*}\|_1+\lambda(1-\alpha) \rho\|\boldsymbol{\beta}^{*}-\tilde{\boldsymbol{\beta}}\|_1-2\lambda(1-\alpha)(1-\rho){\boldsymbol{\beta}^{*}}^{\top}\tilde{\boldsymbol{\beta}} +\lambda(1-\alpha)(1-\rho)\|\tilde{\boldsymbol{\beta}}\|_2^2
    ,
\end{align*}
where $\tilde{\boldsymbol{\varepsilon}}=(\boldsymbol{\varepsilon}^{\top},-{\boldsymbol{\beta}^{*}}^{\top}\mathbf{I}\sqrt{2 n \lambda (1 - \rho)})^{\top} \in \mathbb{R}^{n+p}$ with $\boldsymbol{\varepsilon} = (\varepsilon_1,\ldots,{\varepsilon}_{n}) \in \mathbb{R}^{n}$.

Then, we prove Theorem \ref{thm:estimation_error_tenet} as follows.
\begin{proof}[Proof of Theorem \ref{thm:estimation_error_tenet}]
From Lemma \ref{lem:xe_subgaussian}, for some $c>0$, we have
\begin{align*}
    &\frac{1}{2 n}\|\tilde{\boldsymbol{\varepsilon}}\|_2^2 - \frac{1}{2 n}\|\tilde{\boldsymbol{y}}-\widetilde{\mathbf{X}} \boldsymbol{\beta}^{*}\|_2^2
    \\
    &=
    \frac{1}{2 n}\left\{\|\tilde{\boldsymbol{\varepsilon}}\|_2^2-\left\|\tilde{\boldsymbol{\varepsilon}}-\widetilde{\mathbf{X}}\left(\hat{\boldsymbol{\beta}}-\boldsymbol{\beta}^*\right)\right\|_2^2\right\}
    \\
    &=
    \frac{1}{n}\left(\widetilde{\mathbf{X}}^{\top} \tilde{\boldsymbol{\varepsilon}}\right)^{\top}\left(\hat{\boldsymbol{\beta}}-\boldsymbol{\beta}^*\right)-\frac{1}{2 n}\left\|\widetilde{\mathbf{X}}\left(\hat{\boldsymbol{\beta}}-\boldsymbol{\beta}^*\right)\right\|_2^2
    \\
    &=
    \frac{1}{n}\left({\mathbf{X}}^{\top} {\boldsymbol{\varepsilon}}\right)^{\top}\left(\hat{\boldsymbol{\beta}}-\boldsymbol{\beta}^*\right)
    -
    2\lambda(1-\rho){\boldsymbol{\beta}^{*}}^{\top}\left(\hat{\boldsymbol{\beta}}-\boldsymbol{\beta}^*\right)
    -
    \frac{1}{2 n}\left\|\widetilde{\mathbf{X}}\left(\hat{\boldsymbol{\beta}}-\boldsymbol{\beta}^*\right)\right\|_2^2
    \\
    &\leq
    c\lambda \left\|\hat{\boldsymbol{\beta}}-\boldsymbol{\beta}^*\right\|_{1}
    -
    2\lambda(1-\rho){\boldsymbol{\beta}^{*}}^{\top}\left(\hat{\boldsymbol{\beta}}-\boldsymbol{\beta}^*\right)
    -
    \frac{1}{2 n}\left\|\widetilde{\mathbf{X}}\left(\hat{\boldsymbol{\beta}}-\boldsymbol{\beta}^*\right)\right\|_2^2
    ,
\end{align*}
with probability at least $1-\exp(-n c^2 \lambda^2 /2\sigma^2 + \log(2p))$.
Therefore, it follows that
\begin{align}
    \label{eq:l_minus_l}
    &L(\boldsymbol{\beta}^{*}, \tilde{\boldsymbol{\beta}}) - L(\hat{\boldsymbol{\beta}}, \tilde{\boldsymbol{\beta}})
    \nonumber
    \\
    &\leq
    c\lambda \left\|\hat{\boldsymbol{\beta}}-\boldsymbol{\beta}^*\right\|_{1}
    -
    2\lambda(1-\rho){\boldsymbol{\beta}^{*}}^{\top}\left(\hat{\boldsymbol{\beta}}-\boldsymbol{\beta}^*\right)
    -
    \frac{1}{2 n}\left\|\widetilde{\mathbf{X}}\left(\hat{\boldsymbol{\beta}}-\boldsymbol{\beta}^*\right)\right\|_2^2
    \nonumber
    \\
        &\quad
        +
        \lambda \alpha \rho\|\boldsymbol{\beta}^{*}\|_1+\lambda(1-\alpha) \rho\|\boldsymbol{\beta}^{*}-\tilde{\boldsymbol{\beta}}\|_1-2\lambda(1-\alpha)(1-\rho){\boldsymbol{\beta}^{*}}^{\top}\tilde{\boldsymbol{\beta}}
        \nonumber
        \\
        &\quad
        -\left\{
        \lambda \alpha \rho\|\hat{\boldsymbol{\beta}}\|_1+\lambda(1-\alpha) \rho\|\hat{\boldsymbol{\beta}}-\tilde{\boldsymbol{\beta}}\|_1-2\lambda(1-\alpha)(1-\rho){\hat{\boldsymbol{\beta}}}^{\top}\tilde{\boldsymbol{\beta}}
        \right\}
    \nonumber
    \\
    &\leq
    c\lambda \left\|\hat{\boldsymbol{\beta}}-\boldsymbol{\beta}^*\right\|_{1}
    -
    2\lambda(1-\rho){\left(\boldsymbol{\beta}^{*}-(1-\alpha)\tilde{\boldsymbol{\beta}}\right)}^{\top}\left(\hat{\boldsymbol{\beta}}-\boldsymbol{\beta}^*\right)
    -
    \frac{1}{2 n}\left\|\widetilde{\mathbf{X}}\left(\hat{\boldsymbol{\beta}}-\boldsymbol{\beta}^*\right)\right\|_2^2
    \nonumber
    \\
        &\quad
        +
        \lambda \alpha \rho\|\hat{\boldsymbol{\beta}}_{S}-\boldsymbol{\beta}^{*}_{S}\|_1-\lambda \alpha \rho\|\hat{\boldsymbol{\beta}}_{S^{c}}\|_{1}+\lambda(1-\alpha) \rho\|\boldsymbol{\beta}^{*}-\tilde{\boldsymbol{\beta}}\|_1
        -\lambda(1-\alpha) \rho\|\hat{\boldsymbol{\beta}}-\tilde{\boldsymbol{\beta}}\|_1
        \nonumber
        \\
        &\quad\text{($\because$ From the triangle inequality, $\|\boldsymbol{\beta}_S^*\|_1 \leq \|\hat{\boldsymbol{\beta}}_S-\boldsymbol{\beta}_S^*\|_1+\|\hat{\boldsymbol{\beta}}_S\|_1$.)}
        \nonumber
    \\
    &\leq
    c\lambda \left\|\hat{\boldsymbol{\beta}}-\boldsymbol{\beta}^*\right\|_{1}
    +
    2\lambda(1-\rho){\left\|\boldsymbol{\beta}^{*}-(1-\alpha)\tilde{\boldsymbol{\beta}}\right
    \|_2}\left\|\hat{\boldsymbol{\beta}}-\boldsymbol{\beta}^*\right\|_2
    -
    \frac{1}{2 n}\left\|\widetilde{\mathbf{X}}\left(\hat{\boldsymbol{\beta}}-\boldsymbol{\beta}^*\right)\right\|_2^2
    \nonumber
    \\
        &\quad
        +
        \lambda \alpha \rho\|\hat{\boldsymbol{\beta}}_{S}-\boldsymbol{\beta}^{*}_{S}\|_1-\lambda \alpha \rho\|\hat{\boldsymbol{\beta}}_{S^{c}}\|_{1}+\lambda(1-\alpha) \rho\|\boldsymbol{\beta}^{*}-\tilde{\boldsymbol{\beta}}\|_1
        -\lambda(1-\alpha) \rho\|\hat{\boldsymbol{\beta}}-\tilde{\boldsymbol{\beta}}\|_1
    .
\end{align}
From the definition of $\hat{\boldsymbol{\beta}}$,
we have $L(\boldsymbol{\beta}^{*}, \tilde{\boldsymbol{\beta}}) - L(\hat{\boldsymbol{\beta}}, \tilde{\boldsymbol{\beta}}) \geq 0$.
Therefore, when $\boldsymbol{v}=\hat{\boldsymbol{\beta}} - \boldsymbol{\beta}^{*}$, we obtain
\begin{align*}
    &c\lambda \left\|{\boldsymbol{v}}_{S}\right\|_{1} + c\lambda \left\|{\boldsymbol{v}}_{S^{c}}\right\|_{1}
    +
    2\lambda(1-\rho){\left\|\boldsymbol{\Delta}_{\alpha}\right\|_{2}}\left\|{\boldsymbol{v}}\right\|_{2}
    \\
        &\quad
        +
        \lambda \alpha \rho\|{\boldsymbol{v}}_{S}\|_1-\lambda \alpha \rho\|{\boldsymbol{v}}_{S^{c}}\|_1+\lambda(1-\alpha) \rho\|\boldsymbol{\Delta}\|_1
        -\lambda(1-\alpha) \rho\|{\boldsymbol{v}}-{\boldsymbol{\Delta}}\|_1
    \\
    &\quad\geq 0
    ,
    \\
    \text{i.e., }
    &(\alpha\rho + c)\|\boldsymbol{v}_{S}\|_{1} + (1-\alpha)\rho \|\boldsymbol{\Delta}\|_{1} + 2(1-\rho)\|\boldsymbol{\Delta}_{\alpha}\|_{2}\|\boldsymbol{v}\|_{2}
    \\
    &\geq
    (\alpha \rho - c)\|\boldsymbol{v}_{S^{c}}\|_{1} + (1-\alpha)\rho\|\boldsymbol{v}-\boldsymbol{\Delta}\|_{1}
    .
\end{align*}
This means that $\boldsymbol{v}=\hat{\boldsymbol{\beta}} - \boldsymbol{\beta}^{*} \in \mathcal{B}(\alpha, \rho, c, \boldsymbol{\Delta})$.
From Assumption \ref{ast:gre}, it follows that
\begin{align*}
    \frac{1}{2 n}\left\|\widetilde{\mathbf{X}}\left(\hat{\boldsymbol{\beta}}-\boldsymbol{\beta}^{*}\right)\right\|_2^2
    & =
    \frac{1}{2 n}\left(\hat{\boldsymbol{\beta}}-\boldsymbol{\beta}^{*}\right)^{\top} \widetilde{\mathbf{X}}^{\top} \widetilde{\mathbf{X}}\left(\hat{\boldsymbol{\beta}}-\boldsymbol{\beta}^*\right)
    \\
    & =
    \frac{1}{2} \frac{\left(\hat{\boldsymbol{\beta}}-\boldsymbol{\beta}^*\right)^{\top}\left(\frac{\mathbf{X}^{\top} \mathbf{X}}{n}+2 \lambda (1-\rho) \mathbf{I}\right) \left(\hat{\boldsymbol{\beta}}-\boldsymbol{\beta}^*\right)}{\|\hat{\boldsymbol{\beta}}-\boldsymbol{\beta}^{*}\|_2^2}\|\hat{\boldsymbol{\beta}}-\boldsymbol{\beta}^{*}\|_2^2
    \\
    & \geq
    \frac{1}{2}\{2 \lambda (1-\rho)+\phi\}\|\hat{\boldsymbol{\beta}}-\boldsymbol{\beta}^*\|_2^2
    .
\end{align*}
Therefore, from the inequality (\ref{eq:l_minus_l}), we have
\begin{align*}
    0
    &\leq 
    \lambda\left(\alpha \rho + c\right) \left\|\hat{\boldsymbol{\beta}}-\boldsymbol{\beta}^*\right\|_{1}
    +
    2\lambda(1-\rho){\left\|\boldsymbol{\Delta}_{\alpha}\right\|_{2}}\left\|\hat{\boldsymbol{\beta}}-\boldsymbol{\beta}^*\right\|_{2}
    -
    \frac{1}{2}\{2\lambda (1-\rho)+\phi\}\|\hat{\boldsymbol{\beta}}-\boldsymbol{\beta}^*\|_2^2
    \\
    &\quad
    +\lambda(1-\alpha) \rho\|\boldsymbol{\Delta}\|_1
    \\
    &\leq 
    \left\{\lambda\left(\alpha \rho + c\right)\sqrt{s} 
    +
    2\lambda(1-\rho){\left\|\boldsymbol{\Delta}_{\alpha}\right\|_{2}}\right\}
    \left\|\hat{\boldsymbol{\beta}}-\boldsymbol{\beta}^*\right\|_{2}
    -
    \frac{1}{2}\{2 \lambda (1-\rho)+\phi\}\left\|\hat{\boldsymbol{\beta}}-\boldsymbol{\beta}^*\right\|_2^2
    \\
    &\quad
    +\lambda(1-\alpha) \rho\|\boldsymbol{\Delta}\|_1.
    \\
    &\quad\text{($\because$ From the norm inequality,
    $\|\hat{\boldsymbol{\beta}}_S-\boldsymbol{\beta}_S^*\|_1
    \leq \sqrt{s}\|\boldsymbol{\beta}_S-\boldsymbol{\beta}_S^*\|_2
    \leq \sqrt{s}\|\hat{\boldsymbol{\beta}}-\boldsymbol{\beta}^*\|_2$.)}
\end{align*}
Solving this inequality with respect to $\|\hat{\boldsymbol{\beta}}-\boldsymbol{\beta}^*\|_2$ completes the proof.
\end{proof}

\section{Proof of Proposition \ref{prop:tenet_vs_enet}}

As preparation, we prove the following lemma.

\begin{lemma}
    \label{prop:phi_tenet_enet}
    Suppose that Assumption \ref{ast:gre} is satisfied. Then we have $\phi(\mathcal{B}(\alpha, \rho, c, \boldsymbol{0})) \geq \phi(\mathcal{B}(1, \rho, c, \boldsymbol{0}))$.
\end{lemma}
\begin{proof}
$\mathcal{B}(\alpha, \rho, c, \boldsymbol{0})$ and $\mathcal{B}(1, \rho, c, \boldsymbol{0})$ are, respectively, expressed as 
\begin{align*}
    &\mathcal{B}(\alpha, \rho, c, \boldsymbol{0})
    \\
    &=
    \{\boldsymbol{v} \in \mathbb{R}^p
    : (\alpha \rho - c)\|\boldsymbol{v}_{S^{c}}\|_{1} + (1-\alpha)\rho\|\boldsymbol{v}\|_{1}
    \leq
    (\alpha\rho + c)\|\boldsymbol{v}_{S}\|_{1} + 2\alpha(1-\rho)\|\boldsymbol{\beta}^{*}\|_{2}\|\boldsymbol{v}\|_{2}\}
    ,
\end{align*}
and
\begin{align*}
    \mathcal{B}(1, \rho, c, \boldsymbol{0})
    &=
    \{\boldsymbol{v} \in \mathbb{R}^p
    : (\rho - c)\|\boldsymbol{v}_{S^{c}}\|_{1}
    \leq
    (\rho + c)\|\boldsymbol{v}_{S}\|_{1} + 2(1-\rho)\|\boldsymbol{\beta}^{*}\|_{2}\|\boldsymbol{v}\|_{2}\} 
    .
\end{align*}
For the left side of the inequalities represented in the definition of $\mathcal{B}(\alpha, \rho, c, \boldsymbol{0})$ and $\mathcal{B}(1, \rho, c, \boldsymbol{0})$, 
\begin{align*}
    (\rho - c)\|\boldsymbol{v}_{S^{c}}\|_{1}
    &=
    (\alpha\rho - c)\|\boldsymbol{v}_{S^{c}}\|_{1} + (1-\alpha)\rho \|\boldsymbol{v}_{S^{c}}\|_{1}
    \\
    &=
    (\alpha\rho - c)\|\boldsymbol{v}_{S^{c}}\|_{1} + (1-\alpha)\rho \|\boldsymbol{v}_{S^{c}}\|_{1}
    \\
    &\quad
    +
    (1-\alpha)\rho\|\boldsymbol{v}\|_{1}
    -
    (1-\alpha)\rho\|\boldsymbol{v}\|_{1}
    \\
    &=
    (\alpha\rho - c)\|\boldsymbol{v}_{S^{c}}\|_{1}
    +
    (1-\alpha)\rho\|\boldsymbol{v}\|_{1}
    \\
    &\quad
    -
    (1-\alpha)\rho\left(\|\boldsymbol{v}\|_{1}-\|\boldsymbol{v}_{S^{c}}\|_{1}\right)
    \\
    &\leq
    (\alpha\rho - c)\|\boldsymbol{v}_{S^{c}}\|_{1}
    +
    (1-\alpha)\rho\|\boldsymbol{v}\|_{1}
    .
\end{align*}
Subsequently, for the right side,
\begin{align*}
    (\alpha\rho + c)\|\boldsymbol{v}_{S}\|_{1} + 2\alpha(1-\rho)\|\boldsymbol{\beta}^{*}\|_{2}\|\boldsymbol{v}\|_{2}
    &\leq
    (\rho + c)\|\boldsymbol{v}_{S}\|_{1} + 2(1-\rho)\|\boldsymbol{\beta}^{*}\|_{2}\|\boldsymbol{v}\|_{2}
    .
\end{align*}
Therefore, we have $\mathcal{B}(\alpha, \rho, c, \boldsymbol{0}) \subseteq \mathcal{B}(1, \rho, c, \boldsymbol{0})$ and then, $\phi(\mathcal{B}(\alpha, \rho, c, \boldsymbol{0})) \geq \phi(\mathcal{B}(1, \rho, c, \boldsymbol{0}))$.
\end{proof}

Based on this lemma, we establish the proof of Proposition \ref{prop:tenet_vs_enet} as follows.

\begin{proof}[Proof of Proposition \ref{prop:tenet_vs_enet}]
Suppose $\tilde{\boldsymbol{\beta}}=\boldsymbol{\beta}^{*}$. Then, we have
\begin{align*}
    U_{\mathrm{TENet}}
    &=
    \frac{2(\alpha \rho+c) \lambda \sqrt{s}+4 \lambda \alpha(1-\rho)\left\|\boldsymbol{\beta}^{*}\right\|_2}{2 \lambda (1-\rho)+\phi_{\mathrm{TENet}}}
    ,
    \\
    U_{\mathrm{ENet}}
    &=
    \frac{2(\rho+c) \lambda \sqrt{s}+4 \lambda (1-\rho)\left\|\boldsymbol{\beta}^{*}\right\|_2}{2 \lambda (1-\rho)+\phi_{\mathrm{ENet}}}
    .
\end{align*}
From Lemma \ref{prop:phi_tenet_enet}, 
$1/\phi_{\mathrm{TENet}} \leq 1/\phi_{\mathrm{ENet}}$.
Since $\alpha \leq 1$,
\begin{align*}
    2(\alpha \rho+c) \lambda \sqrt{s}+4 \lambda \alpha(1-\rho)\left\|\boldsymbol{\beta}^{*}\right\|_2
    &\leq
    2(\rho+c) \lambda \sqrt{s}+4 \lambda (1-\rho)\left\|\boldsymbol{\beta}^{*}\right\|_2
    .
\end{align*}
Therefore, we obtain $U_{\mathrm{TENet}} \leq U_{\mathrm{ENet}}$.
\end{proof}

\section{Proof of Proposition \ref{prop:tenet_vs_tlasso}}

We prove Proposition \ref{prop:tenet_vs_tlasso} as follows.

\begin{proof}[Proof of Proposition \ref{prop:tenet_vs_tlasso}]
Suppose $\tilde{\boldsymbol{\beta}}=\boldsymbol{\beta}^{*}$. Then, we have
\begin{align*}
    U_{\mathrm{TENet}}
    &=
    \frac{2(\alpha \rho+c) \lambda \sqrt{s}+4 \lambda \alpha(1-\rho)\left\|\boldsymbol{\beta}^{*}\right\|_2}{2 \lambda (1-\rho)+\phi_{\mathrm{TENet}}}
    ,
    \\
    U_{\mathrm{TLasso}}
    &=
    \frac{2(\alpha+c) \lambda \sqrt{s}}{\phi_{\mathrm{ENet}}}
    .
\end{align*}
Regarding the numerator, from $\sqrt{s}/2\geq\|\boldsymbol{\beta}^{*}_{S}\|_2$,
\begin{align*}
    2(\alpha \rho+c) \lambda \sqrt{s}+4 \lambda \alpha(1-\rho)\left\|\boldsymbol{\beta}^{*}\right\|_2
    &=
    2(\alpha \rho+c) \lambda \sqrt{s}+4 \lambda \alpha(1-\rho)\left\|\boldsymbol{\beta}^{*}_{S}\right\|_2
    \\
    &\leq
    2(\alpha \rho+c) \lambda \sqrt{s}+2 \lambda \alpha(1-\rho)\sqrt{s}
    \\
    &=
    2(\alpha+c) \lambda \sqrt{s}
    .
\end{align*}
Due to $\phi_{\mathrm{TENet}} + 2 \lambda (1-\rho) \geq \phi_{\mathrm{TLasso}}$, we obtain $U_{\mathrm{TLasso}} \geq U_{\mathrm{TENet}}$.
\end{proof}

\section{Proof of Proposition \ref{prop:sample_X_GRE}}

We first introduce the following theorem from \cite{raskutti2010restricted}:

\begin{theorem}[Theorem 1 of \cite{raskutti2010restricted}]
    \label{thm:gaussian_random_ineq}
    For any Gaussian random design $\mathbf{X} \in \mathbb{R}^{n \times p}$ with i.i.d. $\mathcal{N}(0, \mathbf{\Sigma})$ rows, there are universal positive constants $c^{\prime},\,c^{\prime \prime}$ such that
    \begin{align*}
        \frac{\|\mathbf{X} \boldsymbol{v}\|_2}{\sqrt{n}}
        &\geq
        \frac{1}{4}\left\|\mathbf{\Sigma}^{1 / 2} \boldsymbol{v}\right\|_2 - 9 M \sqrt{\frac{\log p}{n}}\|\boldsymbol{v}\|_1 \quad \text {for all }\, \boldsymbol{v} \in \mathbb{R}^p
        ,
    \end{align*}
    where $M := \max _{j=1, \ldots, p} \Sigma_{j j}$, with probability at least $1-c^{\prime \prime} \exp (-c^{\prime} n)$.
\end{theorem}

Then, using this theorem, we prove Proposition \ref{prop:sample_X_GRE} as follows.

\begin{proof}[Proof of Proposition \ref{prop:sample_X_GRE}]
From the definition, all $\boldsymbol{v} \in \mathcal{B}(\alpha, \rho, c, \boldsymbol{\Delta})$ satisfies
\begin{align}
    \label{eq:ineq_B}
    (\alpha \rho - c)\|\boldsymbol{v}_{S^{c}}\|_{1} + (1-\alpha)\rho\|\boldsymbol{v}-\boldsymbol{\Delta}\|_{1}
    &\leq
    (\alpha\rho + c)\|\boldsymbol{v}_{S}\|_{1} + (1-\alpha)\rho \|\boldsymbol{\Delta}\|_{1} + 2(1-\rho)\|\boldsymbol{\Delta}_{\alpha}\|_{2}\|\boldsymbol{v}\|_{2}
    .
\end{align}
From the triangle inequality $\|\boldsymbol{v}-\boldsymbol{\Delta}\|_1 \geq\|\boldsymbol{\Delta}\|_1-\|\boldsymbol{v}\|_1$, it holds that
\begin{align*}
    (\alpha \rho - c)\|\boldsymbol{v}_{S^{c}}\|_{1} + (1-\alpha)\rho \left(\|\boldsymbol{\Delta}\|_1 - \|\boldsymbol{v}\|_1\right)
    &\leq
    (\alpha\rho + c)\|\boldsymbol{v}_{S}\|_{1} + (1-\alpha)\rho \|\boldsymbol{\Delta}\|_{1} + 2(1-\rho)\|\boldsymbol{\Delta}_{\alpha}\|_{2}\|\boldsymbol{v}\|_{2}
    .
\end{align*}
Then, using $\|\boldsymbol{v}\|_1 = \|\boldsymbol{v}_{S}\|_1 + \|\boldsymbol{v}_{S^{c}}\|_1$, we have
\begin{align*}
    (2\alpha \rho - c - \rho)\|\boldsymbol{v}_{S^{c}}\|_{1}
    &\leq
    (\rho + c)\|\boldsymbol{v}_{S}\|_{1} + 2(1-\rho)\|\boldsymbol{\Delta}_{\alpha}\|_{2}\|\boldsymbol{v}\|_{2}
    .
\end{align*}
Consequently, it holds that
\begin{align*}
    (2\alpha \rho - c - \rho)\|\boldsymbol{v}\|_{1}
    &\leq
    2\alpha \rho\|\boldsymbol{v}_{S}\|_{1} + 2(1-\rho)\|\boldsymbol{\Delta}_{\alpha}\|_{2}\|\boldsymbol{v}\|_{2}
    .
\end{align*}
By the norm inequality $\|\boldsymbol{v}_S \|_1 \leq \sqrt{s} \|\boldsymbol{v}_S \|_2 \leq \sqrt{s}\|\boldsymbol{v}\|_2$, we obtain
\begin{align*}
    (2\alpha \rho - c - \rho)\|\boldsymbol{v}\|_{1}
    &\leq
    \left(2\alpha \rho \sqrt{s} + 2(1-\rho)\|\boldsymbol{\Delta}_{\alpha}\|_{2}\right)\|\boldsymbol{v}\|_{2}
    .
\end{align*}
Then, using $2\alpha \rho - c - \rho > 0$, we have
\begin{align}
    \label{eq:ineq_l1norm_l2norm}
    \|\boldsymbol{v}\|_{1}
    &\leq
    \frac{2\alpha \rho \sqrt{s} + 2(1-\rho)\|\boldsymbol{\Delta}_{\alpha}\|_{2}}{2\alpha \rho - c - \rho}\|\boldsymbol{v}\|_{2}
    .
\end{align}
From Theorem \ref{thm:gaussian_random_ineq} and (\ref{eq:ineq_l1norm_l2norm}), for all $\boldsymbol{v} \in \mathcal{B}(\alpha, \rho, c, \boldsymbol{\Delta})$, it holds that
\begin{align*}
    \frac{\|\mathbf{X} \boldsymbol{v}\|_2}{\sqrt{n}}
    &\geq
    \frac{1}{4}\left\|\mathbf{\Sigma}^{1 / 2} \boldsymbol{v}\right\|_2 - 9 M \sqrt{\frac{\log p}{n}} \cdot \frac{2\alpha \rho \sqrt{s} + 2(1-\rho)\|\boldsymbol{\Delta}_{\alpha}\|_{2}}{2\alpha \rho - c - \rho}\|\boldsymbol{v}\|_{2}
    ,
\end{align*}
with probability at least $1-c^{\prime \prime} \exp (-c^{\prime} n)$.
Using $\left\|\mathbf{\Sigma}^{1 / 2} \boldsymbol{v}\right\|_2 \geq \sqrt{\gamma} \|\boldsymbol{v}\|_2$, we obtain
\begin{align*}
    \frac{\|\mathbf{X} \boldsymbol{v}\|_2}{\sqrt{n}}
    &\geq
    \left(\frac{\sqrt{\gamma}}{4} - 9 M \sqrt{\frac{\log p}{n}} \cdot \frac{2\alpha \rho \sqrt{s} + 2(1-\rho)\|\boldsymbol{\Delta}_{\alpha}\|_{2}}{2\alpha \rho - c - \rho}\right)\|\boldsymbol{v}\|_{2}
    .
\end{align*}
Therefore, if the sample size $n$ satisfies
\begin{align*}
    n
    &>
    \frac{(72 M)^2}{\gamma} \left(\frac{2\alpha \rho \sqrt{s} + 2(1-\rho)\|\boldsymbol{\Delta}_{\alpha}\|_{2}}{2\alpha \rho - c - \rho}\right) \log p
    ,
\end{align*}
then we have
\begin{align*}
    \frac{\|\mathbf{X} \boldsymbol{v}\|_2}{\sqrt{n}}
    &\geq
    \frac{\sqrt{\gamma}}{8} \|\boldsymbol{v}\|_{2}
    ,
\end{align*}
with probability at least $1-c^{\prime \prime} \exp (-c^{\prime} n)$.
This completes the proof.
\end{proof}

\section{Proof of Theorem \ref{thm:grouping_effect}}

We establish the following proof of Theorem \ref{thm:grouping_effect}.

\begin{proof}[Proof of Theorem \ref{thm:grouping_effect}]
By differentiating $\mathcal{L}(\boldsymbol{\beta} ; \tilde{\boldsymbol{\beta}})$ with respect to $\beta_j$,
\begin{align*}
    \frac{\partial}{\partial \beta_{j}}\mathcal{L}(\boldsymbol{\beta} ; \tilde{\boldsymbol{\beta}})
    &:= -\frac{1}{n} \sum_{i=1}^n X_{i,j}\left(y_i - \boldsymbol{\beta}^{\top} \boldsymbol{X}_i\right)
    +
    \lambda\alpha\rho \operatorname{sgn}(\beta_j) + 2\lambda\alpha(1-\rho) {\beta}_j
    \\
    &\quad +
    \lambda(1-\alpha)\rho \operatorname{sgn}({\beta}_j - \tilde{\beta}_j) + 2\lambda(1-\alpha)(1-\rho) ({\beta}_j - \tilde{\beta}_j)
    .
\end{align*}    
From the definition of the estimator, $\frac{\partial}{\partial \beta_{j}}\mathcal{L}(\hat{\boldsymbol{\beta}} ; \tilde{\boldsymbol{\beta}})=0$
because of strict convexity of $\mathcal{L}(\boldsymbol{\beta} ; \tilde{\boldsymbol{\beta}})$ for $\rho \neq 1$.
Therefore, if $\hat{{\beta}}_{j}\hat{{\beta}}_{k} > 0$ and $(\hat{{\beta}}_{j} - \tilde{{\beta}}_{j})(\hat{{\beta}}_{k}-\tilde{{\beta}}_{k}) > 0$, then we have
\begin{align*}
    &\frac{\partial}{\partial \beta_{j}}\mathcal{L}(\hat{\boldsymbol{\beta}} ; \tilde{\boldsymbol{\beta}}) - \frac{\partial}{\partial \beta_{k}}\mathcal{L}(\hat{\boldsymbol{\beta}} ; \tilde{\boldsymbol{\beta}})
    \\
    &= -\frac{1}{n} \sum_{i=1}^n \left(X_{i,j} - X_{i,k}\right)\left(y_i - \hat{\boldsymbol{\beta}}^{\top} \boldsymbol{X}_i\right)
    +
    2\lambda(1-\rho) \left(\hat{\beta}_j - \hat{\beta}_k\right)
    \\
    &\quad 
    -2\lambda(1-\alpha)(1-\rho) (\tilde{\beta}_k - \tilde{\beta}_j)
    =0
    ,
\end{align*}
thus we obtain
\begin{align*}
    \left|\hat{\beta}_j - \hat{\beta}_k\right|
    &=
    \frac{1}{2n\lambda(1-\rho)} \left|\left(\boldsymbol{X}_{\cdot,j} - \boldsymbol{X}_{\cdot,k}\right)^{\top}\left(\boldsymbol{y} - \hat{\boldsymbol{\beta}}^{\top} \mathbf{X}\right)
    + (1-\alpha)(\tilde{\beta}_k - \tilde{\beta}_j)\right|
    ,
    \\
    &\leq
    \frac{1}{2n\lambda(1-\rho)} \left\|\boldsymbol{X}_{\cdot,j} - \boldsymbol{X}_{\cdot,k}\right\|_2\left\|\boldsymbol{y} - \hat{\boldsymbol{\beta}}^{\top} \mathbf{X}\right\|_2
    + (1-\alpha)\left|\tilde{\beta}_k - \tilde{\beta}_j\right|
    ,
    \\
    &=
    \frac{1}{2n\lambda(1-\rho)} \left\|\boldsymbol{y} - \hat{\boldsymbol{\beta}}^{\top} \mathbf{X}\right\|_2 \sqrt{2 n (1 - r_{jk})}
    + (1-\alpha)\left|\tilde{\beta}_k - \tilde{\beta}_j\right|
    ,
\end{align*}
where $\boldsymbol{X}_{\cdot,j}$ denotes the $j$th column of $\mathbf{X}$.
Since $\frac{\partial}{\partial \beta_{j}}\mathcal{L}(\hat{\boldsymbol{\beta}} ; \tilde{\boldsymbol{\beta}}) \leq \frac{\partial}{\partial \beta_{j}}\mathcal{L}({\boldsymbol{0}} ; \tilde{\boldsymbol{\beta}})$, we have
\begin{align*}
    &\frac{1}{2n}\left\|\boldsymbol{y} - \hat{\boldsymbol{\beta}}^{\top} \mathbf{X}\right\|_{2}^{2}
    \\
    &\leq
    \frac{1}{2n}\left\|\boldsymbol{y} - \hat{\boldsymbol{\beta}^{\top}} \mathbf{X}\right\|_{2}^{2}
    +
    \lambda \left[ \alpha \left\{\rho \|\hat{\boldsymbol{\beta}}\|_1 + (1-\rho) \|\hat{\boldsymbol{\beta}}\|_2^2\right\}
    +
    (1-\alpha) \left\{\rho \|\hat{\boldsymbol{\beta}} - \tilde{\boldsymbol{\beta}}\|_1 + (1-\rho) \|\hat{\boldsymbol{\beta}} - \tilde{\boldsymbol{\beta}}\|_2^2\right\} \right]
    \\
    &\leq \frac{1}{2n} \left\|\boldsymbol{y}\right\|_{2}^{2}
    +
    \lambda(1-\alpha)\rho \|\tilde{\boldsymbol{\beta}}\|_1 + \lambda(1-\alpha) (1-\rho) \|\tilde{\boldsymbol{\beta}}\|_2^2
    .
\end{align*}
From the results above, we obtain
\begin{align*}
    \left|\hat{\beta}_j - \hat{\beta}_k\right|
    &\leq
    \sqrt{\frac{\left\|\boldsymbol{y}\right\|_{2}^{2}
    +
    2n\lambda(1-\alpha)\rho \|\tilde{\boldsymbol{\beta}}\|_1 + 2n\lambda(1-\alpha) (1-\rho) \|\tilde{\boldsymbol{\beta}}\|_2^2}{2n\lambda^{2}(1-\rho)^{2}}} \sqrt{1 - r_{jk}}
    + (1-\alpha)\left|\tilde{\beta}_k - \tilde{\beta}_j\right|
    .
\end{align*}
\end{proof}

\bibliography{bibliography}
\bibliographystyle{apalike}

\end{document}